\documentclass[a4paper,10pt]{article}
\usepackage[utf8]{inputenc}
\usepackage[T1]{fontenc}
\usepackage{lmodern}

\usepackage{hyperref}
\usepackage{url}
\usepackage{authblk}
\usepackage{natbib}

\usepackage[left=2cm,%
                right=2cm,%
                top=2.25cm,%
                bottom=2.25cm,%
                headheight=12pt,%
                a4paper]{geometry}%

\usepackage{booktabs}       
\usepackage{amsfonts}       
\usepackage{nicefrac}       
\usepackage{microtype}      
\usepackage{xcolor}         
\usepackage{caption,subcaption} 
\usepackage{graphicx}       
\graphicspath{ {figures/} }

\usepackage{wrapfig} 

\usepackage{xspace}
\makeatletter
\DeclareRobustCommand\onedot{\futurelet\@let@token\@onedot}
\def\@onedot{\ifx\@let@token.\else.\null\fi\xspace}

\def\iid{{i.i.d}\onedot}
\def\eg{{e.g}\onedot} 
\def\ie{{i.e}\onedot}

\makeatother

\usepackage{amsfonts,amsmath,amssymb,amsthm}       
\newtheorem{theorem}{Theorem}

\newtheorem{lemma}[theorem]{Lemma}
\theoremstyle{definition}
\newtheorem{definition}{Definition}

\newtheoremstyle{assumption}
  {\topsep}
  {\topsep}
  {\itshape}
  {0pt}
  {\itshape}
  {. ---}
  { }
  {\thmname{#1}\thmnumber{ #2}\textnormal{\thmnote{ (#3)}}}

\theoremstyle{assumption}
\newtheorem{assumption}{Assumption}

\DeclareMathOperator*{\argmin}{\operatorname{argmin}}

\newcommand{\Ccal}{\mathcal{C}}
\newcommand{\E}{\mathbb{E}}

\newcommand{\Gcal}{\mathcal{G}}
\newcommand{\Hcal}{\mathcal{H}}
\newcommand{\Lcal}{\mathcal{L}}
\newcommand{\hatLcal}{\widehat{\Lcal}}

\newcommand{\Pcal}{\mathcal{P}}

\newcommand{\R}{\mathbb{R}}
\newcommand{\Rfrak}{\mathfrak{R}}

\newcommand{\Ucal}{\mathcal{U}} 
\newcommand{\Ufrak}{\mathfrak{U}} 

\newcommand{\Xcal}{\mathcal{X}}
\newcommand{\Ycal}{\mathcal{Y}}
\newcommand{\Zcal}{\mathcal{Z}}

\newcommand{\Nnt}{N_{\text{nt}}} 
\newcommand{\deltaprime}{\delta/\Nnt} 

\newcommand{\ass}{Assumption~\ref{ass:generalization}\xspace}
\newcommand{\asss}{Assumption~\ref{ass:convergence}\xspace}
\newcommand{\bothass}{Assumptions~\ref{ass:generalization} and~\ref{ass:convergence}\xspace}

\newcommand{\myparagraph}[1]{\paragraph{#1}}

\title{Generalization In Multi-Objective Machine Learning}

\author{Peter Súkeník${}^*$}
\author{Christoph H. Lampert${}^*$}
\affil{Institute of Science and Technology Austria (ISTA)\\
       Am Campus 1, 3400 Klosterneuburg, Austria\\
       \texttt{\{peter.sukenik,chl\}@ist.ac.at}}

\date{}

\begin{document}
\renewcommand*{\thefootnote}{*}
\footnotetext{equal contribution}
\renewcommand*{\thefootnote}{\arabic{footnote}}
\maketitle

\begin{abstract}
Modern machine learning tasks often require considering not just one
but multiple objectives. For example, besides the \emph{prediction quality}, 
this could be the \emph{efficiency}, \emph{robustness} or \emph{fairness} of the learned models,
or any of their combinations.
Multi-objective learning offers a natural framework for handling 
such problems without having to commit to early trade-offs.
Surprisingly, statistical learning theory so far offers almost no 
insight into the generalization properties of multi-objective 
learning. 
In this work, we make first steps to fill this gap: we establish 
foundational generalization bounds for the multi-objective setting as 
well as generalization and excess bounds for learning with 
scalarizations.
We also provide the first theoretical analysis of the
relation between the Pareto-optimal sets of the true objectives
and the Pareto-optimal sets of their empirical approximations from 
training data.
In particular, we show a surprising asymmetry: all Pareto-optimal 
solutions can be approximated by empirically Pareto-optimal ones, 
but not vice versa. 
\end{abstract}

\section{Introduction}
Traditionally, statistical machine learning 
has concentrated on solving one single-objective 
optimization problem: to minimize the average 
loss over a given training set. 
Additional quantities of interest, such as 
\emph{model complexity}, had to be either 
addressed implicitly by the choice of model 
class, or integrated into the main objective 
via weighted regularization terms.
Recently, however, additional quantities of 
interest have made it into the focus of the 
machine learning community, such as the 
\emph{fairness}, \emph{robustness},  
\emph{efficiency} or \emph{interpretability}
of the learned models. 
Optimizing these can be in conflict with the 
goal of low training loss and task-specific 
trade-offs need to be made.
Unfortunately, hard-coding such trade-offs can
have undesirable consequences, and model-selecting
them is a cumbersome process when multiple 
objectives are involved. 

To avoid the need for \emph{a priori} 
trade-offs, \emph{multi-objective learning}
has recently received increasing attention. 
Using \emph{multi-objective optimization}, 
it either finds promising trade-off parameters 
at the same time as training the actual model,
or it computes multiple solutions that reflect 
different trade-offs, ideally along the 
complete \emph{Pareto-front}\footnote{We 
define the technical terms \emph{Pareto-front}, 
\emph{Pareto-optimal} and \emph{scalarization} 
in Section~\ref{sec:background}.}
While multi-objective optimization and learning
are algorithmically rich fields, their theory 
is much less well explored. 
In particular, learning-theoretic results, such 
as generalization bounds, are almost completely 
missing.

In this work, we aim at putting multi-objective 
learning on solid theoretic foundations. 
Specifically, we present three results of fundamental 
nature for understanding the properties of learning 
with multiple objectives.
1) We show that generalization bounds of individual 
learning objectives carry over also to the situation 
when learning with multiple objectives simultaneously. 
2) We provide generalization and excess bounds 
that hold uniformly across a broad range of 
\emph{scalarization} techniques. 
3) We analyse in what sense the set of 
models that are \emph{empirically Pareto-optimal} 
(\ie optimal with respect to a training set) 
approximates the set of models that are actually 
\emph{Pareto-optimal} (\ie optimal with respect
to the data distribution). 
Our results provide theoretical justifications for 
the use of scalarization-based as well as Pareto-based 
multi-objective optimization in a learning context,
though with some caveats that have no analog in
single-objective learning.

\begin{figure}[t]
\begin{subfigure}[b]{.32\textwidth}\centering
\includegraphics[width=\textwidth]{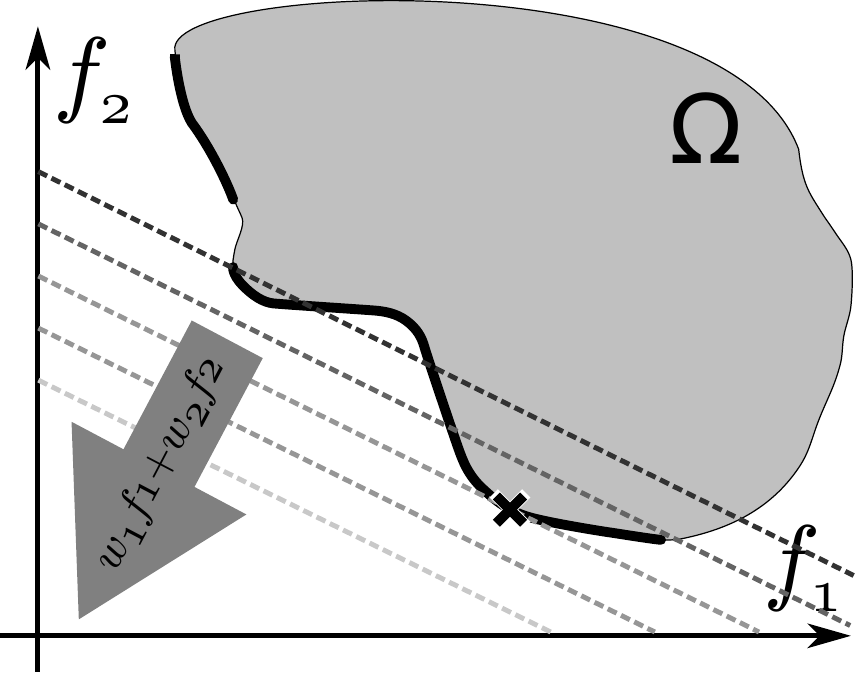}
\caption{Linear (convex) scalarization}\label{fig:convex}
\end{subfigure}
\begin{subfigure}[b]{.32\textwidth}\centering
\includegraphics[width=\textwidth]{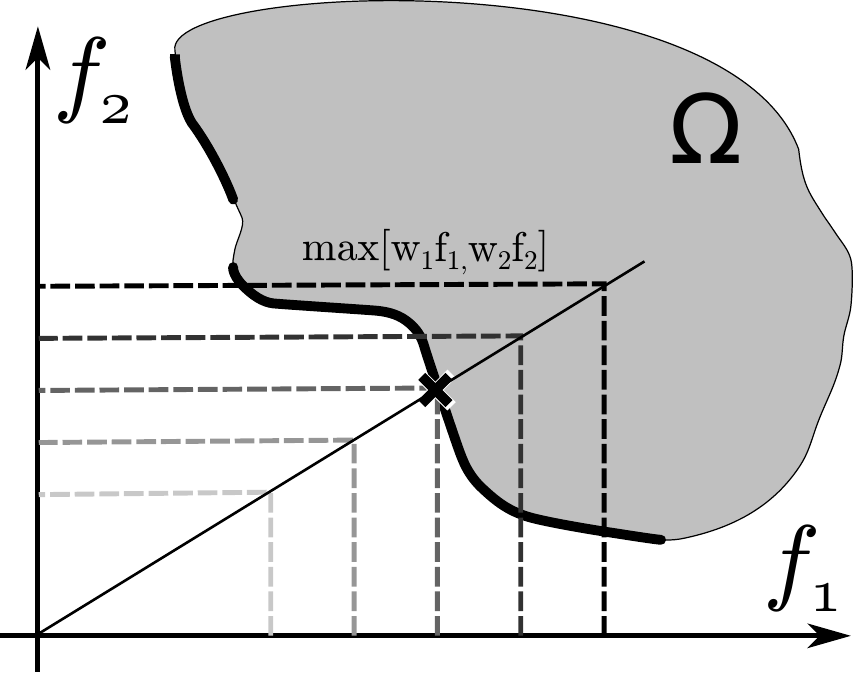}
\caption{Chebyshev scalarization}\label{fig:chebyshev}
\end{subfigure}
\begin{subfigure}[b]{.32\textwidth}\centering
\includegraphics[width=\textwidth]{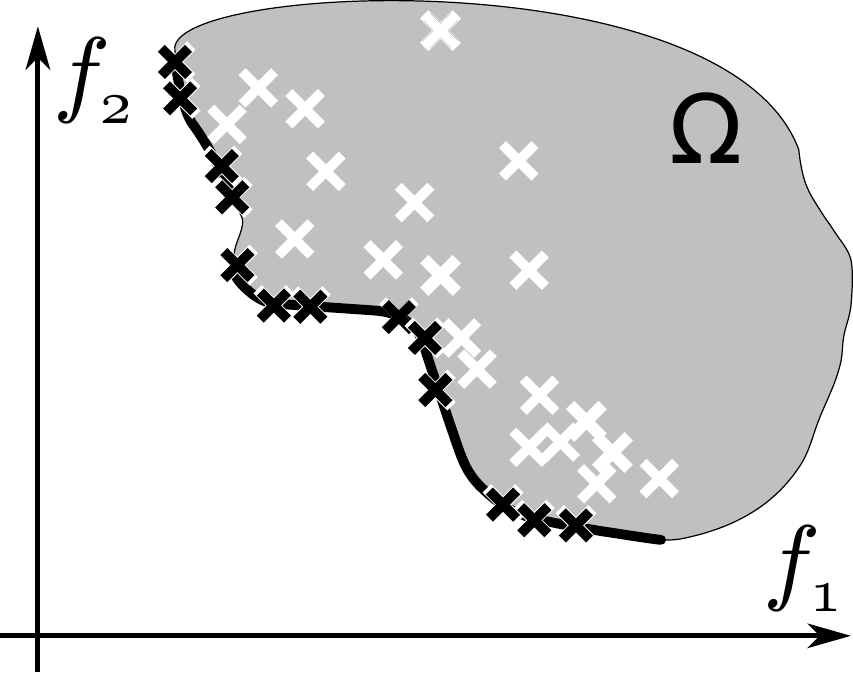}
\caption{Ensemble method}\label{fig:ensemble}
\end{subfigure}
\caption{For general multi-objective optimization problems the Pareto-front (bold curves) 
can be disconnected and non-convex. (\subref{fig:convex}) \emph{Linear scalarization} 
can find Pareto-optimal solutions on the convex hull of the front. 
(\subref{fig:chebyshev})~\emph{Chebyshev scalarization} can find solutions everywhere 
on the front. 
(\subref{fig:ensemble})~\emph{Ensemble methods} compute many solutions,
aiming for the complete Pareto-front to be represented.}
\label{fig:multiobjective}
\end{figure}

\section{Notation and background}\label{sec:background}

In this section, we introduce our notation and 
provide background information on multi-objective 
optimization and learning, as well as statistical 
learning theory.
Our description follows standard textbooks, such 
as~\citet{miettinen2012nonlinear} and \citet{nocedal1999numerical}
for optimization, and~\citet{mohri2018foundations} and \citet{shalev2014understanding}
for machine learning. More details and derivations can be found there.

\subsection{Single- and multi-objective optimization}\label{sec:optimization}
At the heart of most modern machine learning algorithms
lies an optimization step.
In standard (single-objective) optimization, one 
is given an input set, $\Omega$, and an objective 
function, $f:\Omega\to\R$.
Because the objective values are just real 
numbers, they are totally ordered: any 
two point $\omega,\omega'\in\Omega$
are \emph{comparable} in the sense that 
at least one of the relations $f(\omega)\leq f(\omega')$ 
or $f(\omega')\leq f(\omega')$ holds. 
Consequently, it is a natural question to 
ask which $\omega^*\in\Omega$ achieve the 
smallest objective value, if any.
A plethora of \emph{single-objective optimization} 
methods have been developed to answer this question, 
let it be \emph{gradient-based}~\citep{lemarechal2012cauchy,nocedal1999numerical} 
or \emph{derivative-free}~\citep{audet2017derivative,bremermann1962optimization}.

In \emph{multi-objective optimization}, one is 
given multiple objective functions, $f_1,f_2,\dots,f_N:\Omega\to\R$, 
or equivalently, one vector-valued function, $F:\Omega\to\R^N$
with $F(\omega)=\bigl(f_1(\omega),\dots,f_N(\omega)\bigr)$.
We can again define an associated order relation:
\begin{definition}
For $\omega,\omega'\in\Omega$ we say that $\omega$ \emph{weakly dominates} 
$\omega'$ if $f_j(\omega)\leq f_j(\omega')$ for all $j\in[N]$.
We say  $\omega$ \emph{strongly dominates} $\omega'$ if additionally 
$f_j(\omega)< f_j(\omega')$ for at least one $j\in[N]$.
\end{definition}

Because of the multi-dimensional nature, these orderings are only partial. 
There are pairs $\omega,\omega'\in\Omega$ that are \emph{uncomparable}, 
\ie neither $F(\omega)\preccurlyeq F(\omega')$, nor $F(\omega')\preccurlyeq F(\omega)$ holds.
Consequently, in multi-objective optimization it typically makes no 
sense to look for absolute \emph{best} solutions. Instead, one searches 
for \emph{Pareto-optimal} solutions. 

\begin{definition}
A point $\omega^*\in\Omega$ is called 
\emph{Pareto-optimal} if there is no 
other point $\omega\in\Omega$ that \emph{strongly dominates} it.
The set of all Pareto-optimal points is 
called \emph{Pareto-optimal set}. The set of corresponding 
objective value vectors is called \emph{Pareto-front}.
\end{definition}

A large number of algorithms have been developed 
also for multi-objective optimization. 
When trying to find solutions across the complete 
Pareto-front, meta-heuristics such as \emph{evolutionary algorithms}~\citep{zitzler1999multiobjective}
are often employed. 
If a single Pareto-optimal solution suffices, 
\emph{scalarizations} in combination with 
single-objective optimization can be used~\citep{geoffrion1968proper}. 
A \emph{scalarization} function, $\Ucal:\R_+^N\to\R_+$, 
combines the individual objective values into a single one. 
Prominent examples are weighted $p$-norms:
$\Ucal^{(p)}_w(x_1,\dots,x_N) = \bigl(\sum_{i\in[N]} |w_ix_i|^p\bigr)^{1/p}$ 
for $p\in(1,\infty)$, and $\Ucal^{(\infty)}_w(x_1,\dots,x_N) = \max_{i\in[N]} |w_ix_i|$, 
where $w\in W\subset\R^N_+$ is a vector of weights 
that encode a trade-off between the different objectives.

Arguably the most popular choice of scalarization is 
the $L^1$-norm with weights in the probability simplex
$\Delta_N=\{w\in\R^N_+: \sum_i w_i=1\}$. This means,
one forms \emph{convex combinations} of the individual 
objectives~\citep{gass1955computational}.
For any non-zero choice of weights, minimizers of 
this resulting scalarized objective will be 
Pareto-optimal~\citep{geoffrion1968proper}.
However, the set of solutions obtainable 
by varying the weights might not recover
the complete Pareto-front, unless the 
optimization problem is convex~\citep{censor1977pareto}.
In contrast, the choice $p=\infty$ (called 
\emph{weighted Chebyshev norm}) allows 
recovering the complete Pareto front 
when varying the weights in $\Delta_N$~\citep[Chapter~3.4]{miettinen2012nonlinear}.   
Figure~\ref{fig:multiobjective} illustrates these concepts.

\subsection{Single- and multi-objective learning}\label{sec:learning}
Our analysis in this work applies to supervised 
as well as unsupervised learning. 
Therefore, we adopt a notation that allows 
expressing both of these cases in a single 
concise way. 
Let $p(z)$ be a fixed but unknown data 
distribution over a data space $\Zcal$.
We denote by $\Hcal$ a \emph{hypothesis set} 
and $\ell:\Zcal\times\Hcal\to\R_+$ a \emph{loss function}.
For supervised learning with $\Hcal\subset\{h:\Xcal\to\Ycal\}$, 
one uses $\Zcal=\Xcal\times\Ycal$, and  
$\ell(z,h)=L(y,h(x))$, where $L:\Ycal\times\Ycal\to\R_+$ 
measures, \eg, the classification or regression accuracy.
For unsupervised learning, one uses $\Zcal=\Xcal$, 
and $\ell$ measures, \eg, the reconstruction error 
of a clustering or dimensionality reduction step.

\myparagraph{Single-objective learning.}
Standard (single-objective) learning has the 
goal of identifying a hypothesis with small 
\emph{risk} (\emph{expected loss}),
$\Lcal(h)=\E_{z\in\Zcal}[\ell(z,h)]$. 
To approximate this uncomputable quantity, 
the learner uses a \emph{training set}, 
$S=\{z_1,\dots,z_n\}$ to computes the
\emph{empirical risk}, 
$\hatLcal(h)=\frac{1}{n}\sum_{i=1}^n \ell(z_i,h)$.

\emph{Statistical learning theory} studies 
how well the empirical risk 
approximates the true risk 
and under which conditions minimizing 
the (computable) empirical risk is a good 
strategy for finding solution with low true 
risk. 
Many corresponding results are known. 
In particular, under well-understood 
conditions on $\Hcal$ and $S$, one can 
prove that, with high probability over 
the sampling of $S$, the true risk is well 
approximated by the empirical risk,
uniformly across all hypotheses.
Mathematically, such a guarantee has the 
form of a \emph{generalization bound}:
\begin{align}
\forall\delta\in(0,1)\quad\Pr\Bigl\{ \forall h\in\Hcal: |\Lcal(h)-\hatLcal(h)| & \leq \Ccal(n,\Hcal,\delta) \Bigr\} \geq 1-\delta.
\label{eq:single-generalization}
\end{align}
The problem-dependent \emph{generalization term} $\Ccal(n,\Hcal,\delta)$ 
typically consists of a \emph{complexity} component that reflects the 
expressive power of the hypothesis class, and a \emph{confidence} 
component that reflects the uncertainty due to finite sampling effects. 
Ideally, both components will converge to $0$ when the number 
of samples grows to infinity. 

From bounds of the form~\eqref{eq:single-generalization}
one can derive guarantees that, with high probability,
solutions obtained by minimizing the empirical risk have 
close to optimal true risk.
Formally, for $\hat h^*\in\argmin_{h\in\Hcal}\hatLcal(h)$, 
an \emph{excess risk bound} holds:
\begin{align}
\forall\delta\in(0,1)\quad\Pr\Bigl\{ \Lcal(\hat h^*) &\leq \inf_{h\in\Hcal}\Lcal(h) + \Ccal'(n,\Hcal,\delta) \Bigr\} \geq 1-\delta,
\end{align}
where $\Ccal'(n,\Hcal,\delta)$ is another generalization term as above.

\myparagraph{Multi-objective learning.}
In multi-objective learning, multiple target 
objectives, $\Lcal_1,\dots,\Lcal_N$, characterize 
different properties of interest of the hypotheses. 
Estimating them from a (single) dataset yields 
empirical objectives, $\hatLcal_1,\dots,\hatLcal_N$.
In contrast to the single-objective situation
where the objective function is almost always
related to a measure of prediction quality, 
the multi-objective setting provides a principled
framework for expressing also other relevant 
quantities of a machine learning model, such
as \emph{efficiency}, \emph{robustness}, or 
\emph{fairness}.
Consequently, we allow the objectives to also 
have other forms than just expected values over 
per-sample loss functions, and their empirical 
estimates are not restricted to per-sample averages.
As discussed in Section~\ref{sec:optimization},
the multi-objective setting does not induce a 
total ordering of the hypotheses. Consequently,
\emph{a priori} there will be no overall 
\emph{best} hypothesis anymore. 
Instead, there there are two sets of 
Pareto-optimal hypotheses:

\begin{definition}
a) A hypothesis $h\in\Hcal$ is called \emph{empirically Pareto-optimal} 
if it is Pareto-optimal with respect to the multi-objective optimization 
problem of minimizing $\hatLcal_1(h),\dots,\hatLcal_N(h)$ (with are
computed from some training set $S$).
The set of all such hypotheses we call the \emph{empirically Pareto-optimal set}.

b) A hypothesis $h\in\Hcal$ is called \emph{(truly) Pareto-optimal} 
if it is Pareto-optimal with respect to the multi-objective optimization 
problem of minimizing $\Lcal_1(h),\dots,\Lcal_N(h)$.
The set of all such hypotheses we call the \emph{(truly) Pareto-optimal set}.
\end{definition}

Analogously to single-objective learning, we are most interested 
in finding truly optimal hypotheses (here, \eg, the truly 
Pareto-optimal set), as these can be expected to work well on future
data. However, we can only compute solutions to the empirical problem 
(the empirically Pareto-optimal set). If solutions to the latter problem 
approximate the former it is called \emph{multi-objective generalization}.

In recent years, multi-objective learning has 
received increasing attention in the machine 
learning community, and a number of algorithms 
have been proposed for it.
In their easiest form, one simply picks a 
scalarization method and solves the resulting 
single-objective optimization problem with 
fixed scalarization weights or one optimizes
over those as well~\citep{cortes2020agnostic,deist2021multi,fliege2000steepest}.
Alternatively, one can search for hypotheses 
along the complete (empirically) Pareto-front, 
using, \eg, ensemble techniques~\citep{liu1995learning,van1998evolutionary}, 
model conditioning~\citep{ruchte2021scalable},
or hypernetworks~\citep{navon2021learning}. 

Given the long tradition and algorithmic diversity, 
one could expect \emph{multi-objective statistical learning theory} 
also to be a rich field that provides precise 
quantifications of the relations between true 
and empirical objective (generalization bounds), 
as well as relation between the empirical and 
true Pareto-optimal sets (excess bounds).
Surprisingly, this is not the case, and hardly 
any such results exist in the literature. 

\section{Related work}\label{sec:related}
Solving problems with multiple objectives has 
a long tradition in artificial intelligence~\citep{aziz2016computing,purshouse2002multi,rahwan2008pareto,zhou2011multiobjective}, 
game theory~\citep{fudenberg1991game,pardalos2008pareto}, 
and economics~\citep{hochman1969pareto,keeney1993decisions}.
Since the 1990s it has also attracted attention from 
the machine learning community, \eg~\citet{fieldsend2005pareto,goldberg1989genetic,jin2006multi}. 
Existing works predominantly study the 
problem from an algorithmic perspective, 
in particular proposing and analyzing 
new optimization techniques. 
Mirroring the corresponding developments in 
multi-objective optimization, this includes 
methods for efficiently finding individual 
Pareto-optimal solutions, \eg~\citet{cortes2020agnostic,van2014multi,Ye2021multiobjective}, 
as well as exploring the complete Pareto 
front~\citep{jin2008pareto,navon2021learning,przybylski2017multi,ruchte2021scalable,vamplew2011empirical,van2014multi,zhu2019multi}.
Works in both directions implicitly assume that 
better results of the empirical learning task 
should translate to better results on future data.
So far, this \emph{generalization} aspect was 
studied only empirically.
Theoretical results rather focused on the optimization aspect, 
\eg studying \emph{computational complexity}~\citep{teytaud2007hardness,wang2003learning} 
or \emph{convergence rates}~\citep{stark2003rate},
but not statistical generalization.
A notable exception is~\citet{cortes2020agnostic}, 
which we discuss in detail in the Section~\ref{subsec:cortes}.

\section{Main results}\label{sec:main}
In this section we formally state and 
discuss our main results: generalization 
and excess bounds for scalarizations 
and for Pareto-fronts.
For maximal generality, we formulate the results on 
the generic level introduced in Section~\ref{sec:background}. 
We will discuss instantiations that either improve 
over related existing work or provide new insights 
in Sections~\ref{sec:applications} and we provide
a high-level overview of potential additional
applications in Section~\ref{sec:scenarios}.

\myparagraph{Assumptions.}
Because the multi-objective setting strictly 
generalizes the single-objective one, 
multi-objective generalization is not possible 
unless at least single-objective generalization 
holds.
Therefore, for all our results we adopt the 
following assumption.
\begin{assumption}\label{ass:generalization}
For each objective individually a generalization bound of the form \eqref{eq:single-generalization} holds.
\end{assumption}

Note that \ass is technically easy to fulfill, at least 
for bounded objectives, by setting the required 
generalization terms, $\Ccal_i(n,\Hcal,\delta)$ for $i\in[N]$,     
to large enough constants. 
Our results do hold for such a choice, but their 
interpretation would mostly not be very interesting. 
Therefore, whenever we want to interpret results 
in the light of their approximation quality, we 
additionally make the following assumption.
\begin{assumption}\label{ass:convergence}
For each $i\in[N]$ and for each $\delta\in(0,1)$, it holds that $\Ccal_i(n,\Hcal,\delta) \stackrel{n\to\infty}\to 0$.
\end{assumption}

As we detail in Section~\ref{sec:scenarios}, \ass 
and \asss are fulfilled for many quantities of 
interest related to the \emph{accuracy}, \emph{fairness}, 
\emph{robustness} or \emph{efficiency} of 
machine learning systems.
Noteworthy special cases are objectives that are 
data-independent functions of only the hypothesis, 
for example, regularization terms. 
We say that such objectives \emph{generalize trivially},
because they fulfill $\Lcal(h) = \hatLcal(h)$ 
for all datasets and all $h\in\Hcal$, and 
therefore generalization bounds of the form 
\eqref{eq:single-generalization} hold for them 
trivially with 0 as generalization term.

\subsection{Multi-objective generalization}

Our first result states that if generalization 
bounds hold individually for each objective, 
then they hold also jointly in the multi-objective 
setting, where the empirical objectives are computed
from a single dataset, at only a minor loss of confidence. 
\begin{lemma}[Multi-Objective Generalization Bound]\label{lem:generalization}
Let $\Nnt$ be the number of non-trivial objectives. 
Let $S$ be a random dataset of size $n$. 
For each $i\in[N]$, let $\hatLcal_i$ be an empirical 
estimate of $\Lcal_i$ based on a subset $S_i\subset S$ 
of size $n_i$. Then 
it holds with probability at least $1-\delta$, 
\begin{align}
\forall i\in[N], \ \forall h\in\Hcal: |\Lcal_i(h)-\hatLcal_i(h)| \leq \Ccal_i(n_i,\Hcal,\deltaprime).
\end{align}
\end{lemma}
Lemma~\ref{lem:generalization} is in fact a straight-forward 
consequence of \ass, requiring only a union-bound argument 
as proof. 
We state it explicitly nevertheless because it has 
not appeared in this form in the literature so far.

\subsection{Generalization and excess bounds for scalarizations}
A common way for learning in a multi-objective setting 
is by performing single-objective learning for one or 
multiple scalarizations.
To keep the notation concise, for any scalarization 
$\Ucal:\R^N_+\to\R_+$ and $h\in\Hcal$, we abbreviate
$\Lcal_{\Ucal}(h):=\Ucal(\Lcal_1(h),\dots,\Lcal_N(h))$, 
$\hatLcal_{\Ucal}(h):=\Ucal(\hatLcal_1(h),\dots,\hatLcal_N(h))$.

\begin{theorem}[Generalization and Excess Bounds for Scalarizations]\label{thm:scalarization}
Assume the same setting as for Lemma~\ref{lem:generalization}.
Let $\Ufrak=\{\Ucal:\R^N\to\R_+\}$ be a set of scalarizations, 
each of which is $L_{\Ucal}$-Lipschitz continuous with respect 
to some monotonic norm $\|\cdot\|_{\Ucal}$. 
Then, for all $\delta>0$ the following two statements hold 
with probability at least $1-\delta$.

a) For all $\Ucal\in\Ufrak$ and $h\in\Hcal$:
\begin{align}
\big| \Lcal_{\Ucal}(h) - \hatLcal_{\Ucal}(h) \big| &\leq L_{\Ucal}\big\|\bigl(\Ccal_1(n_1,\Hcal,\deltaprime),\dots,\Ccal_N(n_N,\Hcal,\deltaprime)\bigr)\big\|_{\Ucal}.
\label{eq:scalarization-generalization}
\end{align}

b) For all $\Ucal\in\Ufrak$, for all $\hat h_{\Ucal}^*\in\argmin_{h\in\Hcal}\hatLcal_{\Ucal}(h)$, and for all $h\in\Hcal$:
\begin{align}
\Lcal_{\Ucal}(\hat h_{\Ucal}^*)  &\leq  \Lcal_{\Ucal}(h) + 2L_{\Ucal}\big\|\bigl(\Ccal_1(n_1,\Hcal,\deltaprime),\dots,\Ccal_N(n_N,\Hcal,\deltaprime)\bigr)\big\|_{\Ucal}.
\label{eq:scalarization-excess}
\end{align}
\end{theorem}

\textbf{Proof sketch.} We provide the main arguments of the proofs here. The complete proofs are provided in Appendix~\ref{sec:proofs}.
a) The Lipschitz property implies that the difference of scalarized 
objectives is upper bounded by the norm of the differences in 
objective values. By the norm's monotonicity and 
Lemma~\ref{lem:generalization}, this is again bounded 
by the norm of the generalization terms. 
b) from $\hatLcal_{\Ucal}(\hat h_{\Ucal}^*)\leq \hatLcal_{\Ucal}(h)$ 
it follows that $\Lcal_{\Ucal}(\hat h_{\Ucal}^*) -\Lcal_{\Ucal}(h) \leq \Lcal_{\Ucal}(\hat h_{\Ucal}^*) - \hatLcal_{\Ucal}(\hat h_{\Ucal}^*) + \hatLcal_{\Ucal}(h) -\Lcal_{\Ucal}(h)$.
Using a) we can bound the difference between the first two terms as 
well as the difference between the last two terms on the right hand
side each by the norm of the generalization terms.

\myparagraph{Discussion.} 
Theorem \ref{thm:scalarization} establishes \emph{generalization} 
and \emph{excess bounds} for the situation of scalarization-based
multi-objective learning.
Their relevance lies not only in the inequalities 
\eqref{eq:scalarization-generalization} and \eqref{eq:scalarization-excess} 
themselves, which have the standard single-objective form, 
but also in the fact that these hold \emph{uniformly} over all 
scalarizations $\Ucal\in\Ufrak$. 
This implies that 
one can solve an arbitrary number of scalarized 
problems without suffering a loss of confidence 
in the theoretical guarantees.
That is in contrast to other situations of repeated 
learning, \eg hyperparameter-tuning on a validation 
set, where the statistical guarantees deteriorate 
with the number of hypotheses considered, because 
of the \emph{multiple hypothesis testing} phenomenon~\citep[Chapter~11]{shalev2014understanding}.
Despite its simplicity, the theorem improves over 
prior work, \citet{cortes2020agnostic}, which proved 
guarantees that depend on the size of $\Ufrak$. 
For a more detailed discussion see Section~\ref{subsec:cortes}.

\subsection{Pareto excess bounds}
We now provide a formal analysis of the relation 
between the set of Pareto-optimal hypotheses 
and the set of empirically Pareto-optimal hypotheses. 
First, we show that any two elements of the two Pareto-optimal 
sets fulfill an excess-type inequality with respect to at least 
some of the objectives.
\begin{theorem}\label{thm:pareto-single}
Assume the same situation as for Lemma~\ref{lem:generalization}.
Then, for any $\delta>0$, it holds 
with probability at least $1-\delta$:
for all Pareto-optimal $h^*\!\in\!\Hcal$ and 
empirically Pareto-optimal $\hat h^*\!\in\!\Hcal$ 
there exists a non-empty subset $I\subset[N]$, such 
that 
\begin{align}
\forall i\in I:\quad \Lcal_i(\hat h^*) &\leq \Lcal_i(h^*) +
2\Ccal_i(n_i,\Hcal,\deltaprime).\label{eq:pareto-single}
\end{align}
\end{theorem}

\textbf{Proof sketch.} The proof works by contradiction: 
assume that a pair $(h^*,\hat h*)$ exists such that for no 
index set inequality~\eqref{eq:pareto-single} would hold. 
Then, using Lemma~\ref{lem:generalization}, one could show 
that $h^*$ strongly dominates $\hat h^*$ with respect to the 
empirical objectives, which is a contradiction to the optimality 
of $\hat h^*$.
For the formal steps, see Appendix~\ref{sec:proofs}.
Like the sketch, the formal proof does not actually make use of 
the optimality of $h^*$. This implies that Theorem~\ref{thm:pareto-single} 
holds in fact for all $h\in\Hcal$, making it even more apparent 
that excess bound with respect to individual objectives are 
of limited use for studying multi-objective generalization.

For multi-objective learning the most relevant question is if 
there is an analog of Theorem~\ref{thm:pareto-single} for the 
case of $I=[N]$, \ie if by finding the empirical Pareto-curve 
one also approximately recovers the true Pareto-curve with 
respect to \emph{all} objectives.
This is formalized in the following theorem.

\begin{theorem}[Pareto Excess Bound]\label{thm:pareto}
Assume the same setting as for Lemma~\ref{lem:generalization}.
Then, for any $\delta>0$, it holds 
with probability at least $1-\delta$.

a) For all Pareto-optimal $h^*\!\in\!\Hcal$ there 
exists an empirically Pareto-optimal $\hat h^*\!\in\!\Hcal$ with 
\begin{align}
\forall i\in [N]:\quad \Lcal_i(\hat h^*) &\leq \Lcal_i(h^*) +
2\Ccal_i(n_i,\Hcal,\deltaprime).\label{eq:pareto}
\intertext{b) Assume that the Pareto-front is \emph{ray complete},
\ie for all $R\in\{(r_1,\dots,r_N): r_i>0 \text{ for $i\in[N]$}\}$, there exists 
an $h\in\Pcal$ with $\bigl(\Lcal_1(h),\dots,\Lcal_N(h)\bigr)\propto R$.
Then, for all empirically Pareto-optimal $\hat h^*\!\in\!\Hcal$, 
there exists a Pareto-optimal $h^*\!\in\!\Hcal$ with}
\forall i\in [N]:\quad \Lcal_i(\hat h^*) &\leq \Lcal_i(h^*) + 2\Ccal_i(n_i,\Hcal,\deltaprime).
\label{eq:pareto-inv}
\end{align}
\end{theorem}

\textbf{Proof sketch.} 
To prove part a), we make use of the fact that 
$h^*\in\Hcal$ is dominated with respect to the 
empirical objectives by some empirically Pareto-optimal $h^*\in\Hcal$, 
\ie $\hatLcal_i(\hat h^*)\leq \hatLcal_i(h^*)$ for all $i\in[N]$.
Statement~\eqref{eq:pareto} follow by applying 
Lemma~\ref{lem:generalization} to both sides 
of this inequality and rearranging terms.

The main insight for proving part b) 
is that $\hat h^*\in\argmin_{h\in\Hcal}\hatLcal_{\Ucal}(h)$ 
for the Chebyshev scalarization $\Ucal(x_1,\dots,x_N)=\max_{j\in[N]} w_jx_j$
with weights $w_j=\frac{1}{\hatLcal_j(h^*)}$ for $j\in[N]$,
as long as $\hatLcal_i(h^*)>0$ for all $i\in[N]$. 
With $h^*\in\argmin_{h\in\Hcal}\Lcal_{\Ucal}(h)$ 
it follows from Theorem~\ref{thm:scalarization} 
that $\max_{j\in[N]}w_j\Lcal_j(\hat h^*) \leq \max_{j\in[N]}w_j\Lcal_j(h^*) + 2\max_{j\in[N]}w_j\Ccal_j(n_j,\Hcal,\deltaprime)$.
For any $i\in[N]$, it holds that 
$w_i\Lcal_i(\hat h^*)\leq\max_{j\in[N]}w_j\Lcal_j(\hat h^*)$,
and the assumption of ray completeness ensures that,
$w_i\Lcal_i(\hat h^*)=\max_{j\in[N]}w_j\Lcal_j(\hat h^*)$.
In combination, one obtains the same statement as 
\eqref{eq:pareto-inv}, except with a potentially weaker 
generalization term $\frac{2}{w_i}\max_{j\in[N]}w_j\Ccal_j(n_j,\Hcal,\deltaprime)$.
To obtain the desired result, one creates additively shifted 
objective functions that result in a learning setting equivalent 
to the original one, but in which all terms $w_j\Ccal_j(n_j,\Hcal,\deltaprime)$ 
for $j\in[N]$ are identical, such that, in particular, 
$\max_{j\in[N]}w_j\Ccal_j(n_j,\Hcal,\deltaprime)=w_i\Ccal_i(n_i,\Hcal,\deltaprime)$. 
For complete proofs, see Appendix~\ref{sec:proofs}.

\myparagraph{Discussion.}
The theorems in this section clarify the relation 
between the true Pareto-optimal set and its 
empirical counterpart.
When looking a single objective at a time, the 
relation is nearly trivial: Theorem~\ref{thm:pareto-single}
establishes that any empirically Pareto-optimal hypothesis 
is not much worse than any truly Pareto-optimal hypothesis 
with respect to \emph{at least one} of the objectives.

More interesting is the situation when studying all objectives simultaneously. 
Theorem~\ref{thm:pareto} provides a multi-objective
analog of the classical \emph{empirical risk minimization principle}~\citep{vapnik2013nature}.
Solving the empirical multi-objective learning problem
makes sense as a learning strategy, because for every 
truly Pareto-optimal hypothesis there is an 
empirically Pareto-optimal one that has not much 
larger (true) objective values, jointly 
across all of the objectives. 
Reversely, every empirically Pareto-optimal hypothesis 
is not substantially worse than some Pareto-optimal one, 
if we make an additional assumption on the geometry of 
the Pareto-front.

The employed \emph{ray completeness} assumption is quite 
restrictive and we do not expect it to be fulfilled in
most real-world situations.
For example, it is violated already whenever one of 
the objectives is bounded from below by a constant 
bigger than $0$. 

In the two-objectives situation, ray completeness
does hold if the Pareto-front is a continuous curve between 
some point on the $\Lcal_1$-coordinate axis and some point 
on the $\Lcal_2$-coordinate axis, excluding the origin. 
An example where such a situation can happen is a classification 
task in the realizable setting with \emph{classification error} 
and (suitably defined) \emph{computational cost} as objectives. 
For a sufficently rich hypothesis set, the smallest achievable 
error will be a continuous and mononotically decreasing function 
of the specified computational budget. 
Consequenty, the Pareto-front will be a continuous curve 
between a point $(a,0)$, where $a$ is the classification 
error of the classifier with minimal budget, and a 
point $(0,b)$, where $b$ is the smallest computational 
cost for a classifier achieving minimal classification 
error. 
Note that realizability is necessary. Otherwise, the curve 
would still be monotonic, but the second point in the above 
construction would not lie on the $\Lcal_2$-coordinate axis.
Consequently, ray completeness would not be fulfilled.

While sufficient, ray completeness is certainly not a 
necessary condition. For example, if both the real objective 
and the empirical objective are bounded away from zero by
the same constant, substracting this constant from both
objectives could yield a situation that is equivalent in
terms of multi-objective learning, but in which ray 
completeness might be fulfilled. 
Furthermore, from the Theorem's proof one can see that 
a weaker condition would suffice, namely that every ray 
through the empirical Pareto-front front also intersects 
the actual Pareto front. 
This formulation would complicate the condition, though, 
as it introduces a dependence on the dataset and would
nevertheless still not be a mathematically necessary 
condition.
Therefore, we leave the task of identiying a condition 
that is necessary as well as sufficient to future work. 

\medskip
The following theorem shows that some additional assumption
is required for Theorem~\ref{thm:pareto} to hold. 

\begin{theorem}\label{thm:impossibility}
Let $N\geq 2$. Then, for any $C>0$ there exist a learning 
problem that fulfills \bothass with $\Ccal_i(n_i,\Hcal,\delta)=0$ 
for $i\in[N-1]$, but for which with probability at 
least $\frac{1}{2}$ there exists an empirically 
Pareto-optimal $\hat h^*\in \Hcal$, such that 
for all Pareto-optimal $h^*\in\Hcal$, it holds
\begin{align}
\forall i\in [N-1]:\quad \Lcal_i(\hat h^*) > \Lcal_i(h^*) + C.
\end{align}
\end{theorem}

\setlength{\columnsep}{24pt}%
\begin{wrapfigure}[10]{r}{0.4\textwidth}
\centering
\includegraphics[height=.125\textwidth]{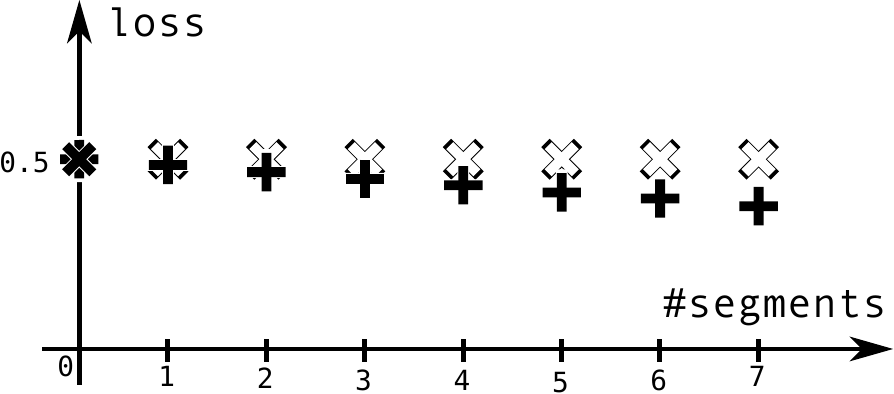}
\caption{Illustration of the counterexample proving Theorem~\ref{thm:impossibility} for $N=2$}
\label{fig:counterexample}
\end{wrapfigure}

\textbf{Proof.} We prove the theorem by constructing a 
concrete counterexample that exploits the classic 
\emph{overfitting} (or \emph{bias-variance trade-off})
phenomenon of single-objective supervised learning~\citep{vapnik2013nature}.

First, we look at the case $N=2$. Let 
$\Zcal=\Xcal\times\Ycal$ with $\Xcal=[0,1]$ and $\Ycal=\{0,1\}$,
$p(x)$ be the uniform distribution and $p(y|x)=\frac12$. 
Let $\Hcal=\{h:\Xcal\to\Ycal\}$ be the set of piecewise-constant 
functions that consist of at most $K$ segments.
We choose the number of jumps as $\Lcal_1$ and the $0/1$-loss as 
$\Lcal_2$. 
Then, \ass and \asss are fulfilled: $\Hcal$ is known to have 
VC-dimension $2k$~\citep{shalev2014understanding}, so a 
classical generalization bound holds for $\Lcal_2$. 
$\Lcal_1$ even generalizes trivially.

We observe that every hypothesis in $h\in\Hcal$ fulfills 
$\Lcal_2(h)=\frac12$. 
Consequently, the two Pareto-optimal solutions, $h$, are 
the constant classifiers, which fulfill $\Lcal_2(h)=\frac12$ 
and $\Lcal_1(h)=0$.

Empirically, however, for sufficiently many points, with high 
probability, the empirical loss $\hatLcal_2(h)$ will be 
strictly monotonically decreasing with respect to $\hatLcal_1(h)$,
as more segments allow to better fit the training data.
Consequently, the set of empirically Pareto-optimal solutions 
will contain elements with $\hatLcal_1(h)=k$ for any $k\in[K]$, 
\ie arbitrarily far from all solutions in the truly Pareto-optimal 
set.
Figure~\ref{fig:counterexample} shows a visualization of this situation.

For larger $N$, we use the analogous construction in $\R^{N-1}$.
Hypotheses have at most $K$ jumps in each coordinate dimension.
Objectives $1$ to $N-1$ are the number of jump per coordinate;
objective $N$ is the classification $0/1$-loss.

\myparagraph{Discussion.}
Theorem~\ref{thm:impossibility} establishes 
that the asymmetry between the statements a) and b) 
of Theorem~\ref{thm:pareto} is an intrinsic property 
of the multi-objective setting, not a limitation 
of our proof techniques.
There can indeed be hypotheses in the empirically 
Pareto-optimal set that are not in an excess 
relation with any hypothesis in the truly 
Pareto-optimal set.
Note that despite the fact that multi-objective 
learning includes single-objective learning as 
a special case, there is no contradiction to
the classical symmetric result. 
For $N=1$, the fact that $I\subset[N]$ is non-empty 
in Theorem~\ref{thm:pareto-single} makes its statement 
identical to Theorem~\ref{thm:pareto} b) without 
the additional assumption. 
Theorem~\ref{thm:impossibility} holds only for $N\geq 2$.

\subsection{Summary} 
In combination, the results of this section
establish a detailed picture of similarities
and differences between the generalization 
properties of single-objective and multi-objective 
learning.
In particular, it highlights a fundamental difference 
between single-objective learning and Pareto-based
multi-objective learning:
in the single-objective setting, empirical risk 
minimization is a good learning strategy, because
with growing data, the every minimizer of the 
empirical risk also has close to optimal true 
risk.
In the multi-objective setting, scalarized learning
has the same properties, but the resulting guarantees
hold only with respect to the scalarization
of the objectives, not each of them individually.
Joint statements across all objectives hold as
well, thereby justifying Pareto-based learning.
However, without additional assumptions excess 
guarantees can only be given for a subset of 
the empirical solutions. 

\section{Applications}\label{sec:applications}
Our results of the previous section 
provide new tools for analyzing learning tasks in 
which multiple, potentially competing, quantities 
are of simultaneous interest, such as \emph{fairness},
\emph{robustness}, \emph{efficiency} and \emph{interpretability}.
So far, the generalization properties of these quantities 
have been studied either not at all, or only with task-specific 
tools. 
Similarly, we expect that \emph{multi-task}, \emph{multi-label} 
and \emph{meta-learning}, as well as \emph{AutoML} will potentially
be able to benefit from the multi-objective view. 

In the rest of this section, we sketch three exemplary applications. 
In Section~\ref{subsec:lasso} we show how our results on empirical 
versus true Pareto-optimality can provide new insights for a 
well-known method. 
In Section~\ref{subsec:term} we demonstrate how our results on
scalarized learning provide a simple and flexible way for constructing 
new single-objective generalization bounds. 
In Section~\ref{subsec:cortes}, we improve an existing generalization 
bound for the multi-objective setting, which thanks to our results 
of Section~\ref{sec:main} requires only a few lines of proof. 
For a more high-level discussion of other application scenarios, 
see Section~\ref{sec:scenarios}.

\subsection{Simultaneous feature selection and regression}\label{subsec:lasso}
The classical LASSO method~\citep{Tibshirani1996Regression} learns 
a linear regression function by solving the following regularized
risk minimization problem 
\begin{align}
\min_{\beta\in\mathbb{R}^d}\quad \frac{1}{n}\sum_{i=1}^n (y_i - \beta^\top x_i)^2 + \lambda \|\beta\|_{L^1}.  \label{eq:LASSO}
\end{align}
Here $\{(x_1,y_1),\dots,(x_n,y_n)\}\subset\R^d\times\R$ 
is a given training set and $\lambda\in\R^+$ is a free parameter 
that trades off the data fidelity on the training set with 
the sparsity of the coefficient vector $\beta$ at its optimum. 
LASSO is particularly popular because it not only learns 
a regression function but also performs feature selection 
and therefore can give rise to more interpretable regression 
models than other regression techniques.
The set of all solutions obtained by minimizing \eqref{eq:LASSO} 
for different values of $\lambda$ is called the \emph{solution path}.
A number of efficient techniques for computing it have been 
developed~\citep{efron2004least,gaines2018algorithms,osborne2000lasso}.

We can interpret the LASSO problem equivalently as the 
linear scalarization of a two-objective learning problem. 
The first objective is the expected squared loss
$\Lcal_1(\beta)=\E_{(x,y)}(y-\beta^\top x)^2$, which 
has $\hatLcal_1(\beta)=\frac{1}{n}\sum_{i=1}^n(y_i-\beta^\top x_i)^2$
as empirical counterpart.
When the data and coefficient vector come from bounded domains, 
standard generalization bounds are known to hold that 
relate $\Lcal_1$ and $\hatLcal_1$, see \eg~\citet[Theorem 11.11]{mohri2018foundations}. 
The second objective is the regularizer, $\Lcal_2(\beta)=\hatLcal_2(\beta)=\|\beta\|_{L^1}$,
which generalizes trivially.

The (single-objective) generalization properties 
of LASSO's squared loss term are well understood. 
The multi-objective view, however, adds insight 
into its relation with the regularizer, which reflects 
the sparsity and thereby the interpretability of 
the solutions.
First, we observe that the underlying optimization 
problem is convex, so each empirically Pareto-optimal 
solution can be recovered by solving~\eqref{eq:LASSO} 
for some value of $\lambda$. 
Therefore, existing \emph{solution path} methods 
can readily be used to identify the empirical Pareto 
front. 

Theorem~\ref{thm:pareto}~a) now ensures that each 
truly Pareto-optimal solution can be approximately 
recovered this way. 
This means, we can be sure that no solutions exist 
that are substantially sparser at equal accuracy or 
more accurate at identical sparsity with respect to 
the true objectives than some in the solution path.

However, Theorem~\ref{thm:impossibility} reminds us 
that not all solutions found by solving~\eqref{eq:LASSO} 
will necessarily be close to truly Pareto-optimal. 
In particular, this means, while each individual element 
of the solution path will have optimal sparsity for 
the empirical accuracy it achieves, its sparsity might 
be far from optimal compared to other solutions of similar 
true accuracy. 
Consequently, if optimal sparsity is important for the 
task at hand, the solutions on the regularization path 
should be further evaluated, \eg using validation data. 

The latter comment does not apply if the true underlying 
regression task is actually linear, such that a coefficient 
vector, $\beta$, exists with vanishing objective, $\Lcal_1(\beta)=0$.
Because the same property holds for $\Lcal_2$ (trivially 
achieved by $\beta=0$), and the regulariation path is a 
connected set~\citep{tibshirani2013lasso}, the condition 
of \emph{ray completeness} would be fulfilled. 
Theorem~\ref{thm:pareto}~b) then guarantees that in 
fact all empirically Pareto-optimal solutions are also 
approximately truly Pareto-optimal.

\subsection{Tilted empirical risk minimization}\label{subsec:term}
Tilted empirical risk minimization (TERM)~\citep{li2021tilted} 
has recently been proposed as a widely applicable technique for 
making learning problem more \emph{robust} or more \emph{fair}.
In its group-based form, TERM consists of minimizing an 
exponentially weighted risk functional
\begin{align}
J_t(f) &= \frac{1}{t}\log\Big(\frac{1}{N}\sum_{i=1}^N e^{tR_i(f)}\Big)
\quad\text{with}\quad
R_i(f) = \frac{1}{n}\sum_{(x,y)\in S_i} \ell(y,f(x))
\label{eq:gTERM}
\end{align}
for a loss function $\ell$ and training data given as $N$ 
potentially overlapping groups, $S_1,\dots,S_N$. For simplicity 
of explosition we assume all groups to be of identical sizes, $n$. 
The tilt parameter $t\in\mathbb{R}\setminus\{0\}$ determines 
whether the effect of TERM is to provide \emph{robustness} 
against outlier groups ($t<0$), or to encourage \emph{fairness}
between all groups ($t>0$). For $t\to\infty$ and $t\to-\infty$, 
TERM converges to $\max_i R_i(f)$ and $\min_i R_i(f)$, respectively.

Taking a multi-objective perspective, $J_t$ in \eqref{eq:gTERM} 
can be seen as a parametrized family of scalarizations of empirical 
objectives $R_1,\dots,R_N$. Each $J_t$ is $1$-Lipschitz with 
respect to the $L^\infty$ norm. 
Assuming that generalization bounds hold for each individual 
group, then Theorem~\ref{thm:scalarization} guarantees that 
a generalization bound of the same structure holds also for 
$J_t$, simultaneously across all values of $t$. 

\emph{Hierarchical TERM (hTERM)}~\citep{li2021tilted} uses the 
same exponentially weighted functional $J_t$, but the per-group 
terms $R_1,\dots,R_K$ that it combines are not averages 
across samples as in~\eqref{eq:gTERM}, but TERM-losses 
themselves with individual tilt parameters $\tau_1,\dots,\tau_K$.
From our previous analysis, we know generalization bounds 
for each of those. Consequently, by the same construction
as above, we readily obtain a generalization bound for 
hTERM, which in fact holds uniformly across all combinations
of tilt parameters.

\subsection{Agnostic learning with multiple objectives}\label{subsec:cortes}
\citet{cortes2020agnostic} studies the generalization properties 
of multi-objective learning for the case of a special scalarization 
obtained by minimizing over convex combinations. 
To allow for an easier comparison, we state their result in 
our notation.\footnote{Our formulation also has slightly 
different constants in the generalization term, which we 
believe to be necessary based on the theorem's proof.}

\begin{theorem}[Theorem~3 in~\citet{cortes2020agnostic}]\label{thm:cortes}
Let $\Hcal$ be a hypothesis set for a supervised learning
problem with input set $\Xcal$ and output set $\Ycal$,
that fulfills $\|(h(x'),y')-(h(x),y)\|\leq D$ for 
some constant $D>0$ and for all $(x,y),(x',y')\in\Xcal\times\Ycal$.
Let $\ell_i:\Ycal\times\Ycal\to\R_+$ for $i\in[N]$ be loss 
functions that are $M_i$-Lipschitz and upper-bounded by $M$. 
Set $\Lcal_i(h)=\E_{(x,y)}[\ell(y,h(x)]$ and $\hatLcal_i(h)=\frac{1}{n}\sum_{i=1}^n\ell(y_i,h(x_i))$
for a dataset $S\stackrel{\iid}{\sim}p(x,y)$ of size $n$. 
For a set of scalarization weights, $W\subset\Delta_N$, let 
$\Lcal_W(h) = \max_{w\in W}\sum_{i=1}^N w_i\Lcal_i(h)$
and 
$\hatLcal_W(h) = \max_{w\in W}\sum_{i=1}^N w_i\hatLcal_i(h)$.
Assume that $\sum_{i=1}^N w_i M_i\leq \beta$ for all $w=(w_1,\dots,w_N)\in W$.

Then, for any $\epsilon>0$ and $\delta\in(0,1)$, with probability
at least $1-\delta$, the following inequality holds 
for all $h\in\Hcal$:
\begin{align}
\Lcal_W(h) &\leq  \hatLcal_W(h) + 2\beta\hat\Rfrak_S(\Hcal) 
           + M\epsilon + 3\beta D\sqrt{\frac{1}{2n}
           \log\Big[\frac{2|W_{\epsilon}|}{\delta}\Big]},
\label{eq:cortes}
\end{align}
where $\hat\Rfrak_S(\Hcal)$ is the \emph{empirical Rademacher
complexity} of the hypothesis class $\Hcal$ with respect to $S$, 
and $|W_{\epsilon}|$ is the size of a minimal $\epsilon$-cover of $W$.
\end{theorem}

One can see that inequality~\eqref{eq:cortes} 
precisely matches the form of our excess bound 
in Theorem~\ref{thm:scalarization} for a 
specific scalarization. 
Indeed, we can derive an analogous theorem 
using our results of Section~\ref{sec:main}. 

\begin{theorem}
\label{thm:analog-of-cortes}
Make the same assumptions as in Theorem~\ref{thm:cortes}. 
Then, for any $\delta\in(0,1)$, with probability at least $1-\delta$ the following inequality holds for all $h\in\Hcal$:
\begin{align}
\Lcal_W(h) &\leq \hatLcal_W(h) + 2\beta\hat\Rfrak_S(\Hcal) + 3\beta D\sqrt{\frac{\log\frac{2N}{\delta}}{2n}}.
\label{eq:analog-of-cortes}
\end{align}
\end{theorem}

\textbf{Proof.} The Lipschitz assumptions imply standard 
Rademacher-based generalization bounds for each 
objective individually~\cite[Theorem~11.3]{mohri2018foundations}: 
with probability at least $1-\delta$ it holds for all $h\in\Hcal$:
\begin{align}
\Lcal_i(h) &\leq  \hatLcal_i(h) + 2M_i\hat\Rfrak_S(\Hcal) 
           + 3M_iD\sqrt{\frac{\log\frac{2}{\delta}}{2n}}.
\end{align}
We now apply Theorem~\ref{thm:scalarization} for the 
family of linear scalarizations, $\Ucal_w(x_1,\dots,x_N)=\sum_{i=1}^N w_ix_i$
with $w\in W$, and we insert that $w_i M_i\leq \beta$.
Finally, taking the maximum over $w$ on both sides of 
the inequality yields \eqref{eq:analog-of-cortes}.

Because of the power of the introduced 
multiobjective framework, the proof of Theorem~\ref{thm:analog-of-cortes}
is much shorter than the original one of Theorem~\ref{thm:cortes}. Nevertheless, 
our result has a number of advantages. 
First, our bound is structually simpler. It holds without 
need for an $\epsilon$-parameter that additively enters the 
right hand side of~\eqref{eq:cortes}, yet also influences the size 
of the right-most confidence term.
Second, the right hand side of our bound is independent of 
the size of $W$, with the confidence term only depending 
on the number of objectives.
As a consequence, our bound is substantially tighter, except 
for trivially small sets $W$. 
For example, for the common case of convex combinations, 
$W=\Delta_N$, the covering size $|W|_{\epsilon}$ is
of order $(1/\epsilon)^{N-1}$. 
This makes the generalization term in \eqref{eq:cortes} of 
order $\sqrt{N/n}$, 
indicating that to preserve confidence the amount of 
data has to grow linearly with the number 
of objectives considered. 
In contrast, the right hand side of our bound 
\eqref{eq:analog-of-cortes} is independent of 
the size of $W$ and its confidence term grows 
only logarithmically with respect to $N$.
Finally, our proof is not only 
simpler than the original one but also more flexible:
it readily extends to other generalization bounds 
rather than just Rademacher-based ones, and to other 
scalarization besides linear combinations.

\section{Further application scenarios}\label{sec:scenarios} 
In Section~\ref{sec:applications} we highlighted 
some specific examples in which our proposed 
multi-objective generalization theory provides 
new insights into existing methods. 
In this section, we provide more high-level background 
and discuss additional quantities that we believe will 
or will not benefit from a multi-objective analysis.

\myparagraph{Fairness.}
\emph{Algorithmic (group) fairness} asks to create classifiers 
that are not only accurate but also do not discriminate 
against certain protected groups in their decisions. 
Formally, this property can be expressed by different 
\emph{(un)fairness measures}, such as \emph{demographic parity}, 
\emph{equality of opportunity} or \emph{equalized odds}~\citep{barocas-hardt-narayanan}.
Because accuracy and fairness can be in conflict 
with each other, fairness-aware learning is a 
prototypical candidate for multi-objective learning~\citep{martinez2020minimax,wei2020fairness,kamani2021pareto,kirtan2021addressing}.
This view also extends naturally to integration of 
multiple fairness measures~\citep{liu2020accuracy}, 
which might be incompatible with each 
other~\citep{kleinberg2016inherent,chouldechova2017fair,berk2021fairness}. 
Generalization bounds for the empirical estimation  
of unfairness measures have been developed~\citep{woodworth2017learning,konstantinov2021fairness}.
Consequently, our results from Section~\ref{sec:main}
apply, yielding a unified understanding of the 
generalization properties of fairness-aware learning, 
\eg, regularization-based~\citep{kamishima2011fairness}
constraint-based~\citep{calders2009parity}, or 
Pareto-based~\citep{liu2020accuracy,navon2021learning}.
The multi-objective view also allows us to conjecture 
that methods that seek fair hypotheses by other 
means, such as pre-processing~\citep{kamiran2012data} 
or post-processing~\citep{hardt2016equality}, might not 
reach (empirically) Pareto-optimal solutions. 
If generalization guarantees do actually hold for 
these, other ways for proving them would be required.

\myparagraph{Robustness.}
It has been observed that deep network classifiers 
in continuous domains such as image classification
are susceptible to \emph{adversarial examples}, \ie
they are not robust against small perturbation 
of the input data. 
Two main research directions have emerged to overcome 
this limitation: 
\emph{Adversarial training}~\citep{madry2018towards} adds a robustness-enforcing 
loss term to the training problem. Generalization 
bounds for such terms have been derived, \eg~\citet{yin2019rademacher}. 
Consequently, multi-objective learning can be used 
in this setting with the guarantees and caveats 
discussed above. 
\emph{Lipschitz-networks}~\citep{cisse2017parseval} 
restrict the hypothesis class to functions with a 
small Lipschitz constant, typically $1$. Afterwards
one solves a training problem that tries to enforce
a large margin between the predicted class label 
and the runner-up. 
From the achieved margins one can infer how large an 
input perturbation the classifier can tolerate without 
changing its decision~\citep{weng2018towards}. 
We are not aware of existing theoretical studies 
of such \emph{certified robustness} techniques. 
However, margin-based loss functions have a long tradition 
in machine learning, and a number of generalization bounds 
exist which are applicable in the described 
situation, such as~\citet{kuznetsov2015rademacher,koltchinskii2002empirical}.

\myparagraph{Efficiency.}
Large machine learning models, in particular deep 
networks, often have high computational demands, 
not only at training but also at prediction 
time~\citep{strubell2020energy,menghani2021efficient}. 
Consequently, a number of techniques have been 
developed that aim at reducing the computational 
cost. 
\emph{Parameter sparsification}~\citep{hoefler2021sparsity} 
and \emph{quantization}~\citep{gholami2021survey} are
widely used methods for reducing the number of 
operations required to evaluate a model. 
As data-independent properties they can readily be 
used as \emph{trivially-generalizing} objectives in 
a multi-objective learning framework~\citep{zhu2019multi}. 
Alternatively, speedup can also be achieved by encouraging
as many zero values as possible to occur as part of the 
internal computation steps of a deep network.
Such \emph{activation sparsity}~\citep{kurtz2020inducing} 
is a data-dependent quantity that can also be shown to 
generalize using standard techniques. 
Therefore it as well can be handled in a multi-objective 
way.
\emph{Adaptive computation} methods, such 
as \emph{ensembles}~\citep{schwing2011adaptive,lampert2012dynamic}, 
\emph{classifier cascades}~\citep{viola2001rapid} or 
\emph{multi-exit architectures}~\citep{huang2018multiscale,teerapittayanon2016branchynet}, 
evaluate different subsets of a larger model depending 
on the input sample. 
For suitable design choices, generalization 
bounds for the resulting computation time can
be proven, and our results will apply. 

\myparagraph{Multi-task and multi-label learning.}
\emph{Multi-task learning} has recently been put 
forward as a multi-objective task, where each 
task's loss is treated as a separate objective~\citep{sener2018multi,lin2019pareto,ma2020efficient,mahapatra2020multi}. 
This setting is of a non-standard form, as each 
task typically has a dedicated training set, and
objectives are not necessarily competing with 
each others~\citep{ruchte2021multi}. 
Nevertheless, our framework can handle this setting 
as well, making use of the property that we allow 
the empirical estimates of different objectives 
to be derived from different subsets of the 
available data. 
Pareto-based guarantees are particularly relevant 
then, because at prediction time, for each sample 
one is interested in only one of the objectives, 
namely the one of the task to which this sample belongs. 
In the related problems of 
\emph{multi-label learning}~\citep{zhang2013review}
and \emph{extreme classification}~\citep{varma2019extreme}, 
the goal is to predict multiple outputs (labels) 
for each sample.
Each label has an associated classifier objective, and 
the losses are estimated either from the total dataset 
or from (typically overlapping) subsets~\citep{shi2012multi}.
Again, our framework is flexible enough to handle this setting.
At prediction time all labels are meant to be predicted, 
and the quality is typically judged by a task-dependent 
aggregate measure, making scalarization approaches of 
particular interest in this setting.

\myparagraph{Limitations.}
Despite its generality, some multi-objective 
learning settings do not lend themselves to 
an analysis using our results.
For example, in the \emph{learning-to-rank} 
setting~\citep{liu2009learning} solutions 
are typically judged by two measures: 
\emph{precision} and \emph{recall}. \emph{A priori}, 
this makes it a promising setting for multi-objective 
analysis~\citep{cao2020ranking,svore2011learning}.
Unfortunately, we are not aware of generalization 
bounds for the \emph{precision} objective. Given 
that its value fluctuates heavily in the low-recall 
regime, it is in fact possible that \asss might not 
be fulfillable.
Also in the context of ranking, two other common objectives are 
\emph{true positive rate (TPR)} and \emph{false positive rate (FPR)}, 
which together trace out the \emph{receiver operating characteristic (ROC) curve}.
TPRs and FPRs can summarized into a single value by the 
\emph{area under the ROC curve (AUC)}~\citep{hanley1982meaning}, 
for which indeed generalization bounds have been 
derived~\citep{agarwal2005generalization}.
However, the AUC is not a scalarization in the sense of 
Section~\ref{sec:optimization}, so Theorem~\ref{thm:scalarization} 
does not apply to it.
Finally, besides the uniform generalization bounds of \ass, 
other guarantees of generalization have been developed, 
\eg, based on PAC-Bayesian theory~\citep{dziugaite2017computing,mcallester1999some}, 
or \emph{algorithmic stability}~\citep{bousquet2002stability,hardt2016train}. 
We see no principled reasons why results similar to ours
should not hold for such settings as well, but other 
techniques would be required that lie outside of 
the scope of this work.

\section{Conclusion}\label{sec:conclusion}
In this work, we proved a number of foundational 
results for the generalization theory of 
multi-objective learning.
In particular, we showed that generalization 
bounds for the individual objectives imply 
generalization and excess bounds for 
multi-objective learning using scalarizations. 
Our second main result is an analysis
of the relation between the Pareto-optimal
sets of the empirical and the 
true learning problem. This justifies 
the use of Pareto-based methods on empirical
data to approximately find all truly 
Pareto-optimal solutions. However, there
is a caveat that some of the solutions 
found might be close to Pareto-optimal 
ones only with respect to some of the 
objectives, not all of them.

We formulated our results on a high
level of generality that applies not only 
to measures of per-sample prediction
quality, for which generalization bounds
were originally developed, but also 
many other quantities of interest
for modern machine learning systems,
such as \emph{fairness}, \emph{robustness},
and \emph{efficiency}.
While initial results for some of these 
specific domains exist, we expect that 
more and stronger guarantees will be 
possible by more refined objective-specific 
analyses.

On a technical level, we see two directions
for potentially improving our results.
First, it would be desirable to have an 
explicit rather than implicit relationship 
between Pareto-optimal hypotheses and their 
best empirically Pareto-optimal approximations. 
Theorem~\ref{thm:pareto} does not provide 
this. Even though its proof contains an 
explicit procedure, it relies on uncomputable 
quantities, such as the true objective 
objective values. 
Second, given that Theorem~\ref{thm:impossibility}
establishes that there can be empirically 
Pareto-optimal hypotheses that do not 
approximate any truly Pareto-optimal 
hypothesis with respect to all objectives,
it would be desirable to have an algorithmic 
procedure for testing which hypotheses 
these are.
We see these as interesting directions for
future work.

\bibliographystyle{authordate1}
\bibliography{ms}

\clearpage
\appendix
\section{Appendix -- proofs of the main results}\label{sec:proofs}

\myparagraph{Proof of Theorem~\ref{thm:scalarization}.}
With probability at least $1-\delta$ for the dataset $S$ 
the relations of Lemma~\ref{lem:generalization} will hold. 
By studying only these cases, we again obtaining 
results that hold with probability at least $1-\delta$. 

For statement a), for any $\Ucal\in\Ufrak$ we obtain by the 
Lipschitz property of the scalarization and Lemma~\ref{lem:generalization} 
that for all $h\in\Hcal$:
\begin{align}
|\Lcal_{\Ucal}(h) - \hatLcal_{\Ucal}(h)|
&\leq  L_{\Ucal}\|\bigl(\Lcal_1(h)-\hatLcal_1(h),\dots,\Lcal_N(h)-\hatLcal_N(h)\bigr)\|_{\Ucal}
\\
&\leq  L_{\Ucal}\|\bigl(|\Lcal_1(h)-\hatLcal_1(h)|,\dots,|\Lcal_N(h)-\hatLcal_N(h)|\bigr)\|_{\Ucal}
\\
&\leq L_{\Ucal}\|\bigl(\Ccal_1(n_1,\Hcal,\deltaprime),\dots,\Ccal_N(n_N,\Hcal,\deltaprime)\bigr)\|_{\Ucal}
\label{eq:lipschitz}
\end{align}
where the last two inequalities hold because of 
the norms' monotonicity, \ie the fact that it 
is non-decreasing under increases of the 
input vector components~\citep{bauer1961absolute}.
In combination, this proves statement a).

Statement b) follows by arguments mirroring the proof of classic 
\emph{excess risk bounds}~\citep{mohri2018foundations}. 
Let $\hat h_{\Ucal}^*\in\argmin_{h\in\Hcal}\hatLcal_{\Ucal}(h)$.
Then, it holds for arbitrary $h\in\Hcal$ that 
\begin{align}
\Lcal_{\Ucal}(\hat h_{\Ucal}^*) -\Lcal_{\Ucal}(h)
&\leq 
\Lcal_{\Ucal}(\hat h_{\Ucal}^*) 
-
\hatLcal_{\Ucal}(\hat h_{\Ucal}^*) 
+
\hatLcal_{\Ucal}(h) 
-\Lcal_{\Ucal}(h)
\\
&\leq
2\|\bigl(\Ccal_1(n_1,\Hcal,\deltaprime),\dots,\Ccal_N(n_N,\Hcal,\deltaprime)\bigr)\|_{\Ucal}
\end{align}
where the first inequality holds because 
$\hatLcal_{\Ucal}(\hat h_{\Ucal}^*) \leq \hatLcal_{\Ucal}(h)$ by construction
of $\hat h_{\Ucal}^*$, and the second inequality one follows 
from applying~\eqref{eq:scalarization-generalization}
twice, once for $h$ and once for $\hat h_{\Ucal}^*$.
The statement of the theorem now follows by moving 
the term containing $h$ to the right hand side. 

\myparagraph{Proof of Theorem~\ref{thm:pareto-single}.}
We again only study the case in the inequalities of 
Lemma~\ref{lem:generalization} are fulfilled, so the 
results we achieve hold with probability at least 
$1-\delta$. 

We prove the remaining part of the theorem by contradiction.
The negation of the statement reads: 
\emph{there exists an empirically Pareto-optimal hypothesis 
$\hat h^*\in\Hcal$ and a hypothesis $h\in\Hcal$ 
such that $\Lcal_i(\hat h^*) - \Lcal_i(h) > 2\Ccal_i(n_i,\Hcal,\deltaprime)$
for all $i\in[N]$.}

For these $\hat h^*$ and $h$ it follows that for all $i\in[N]$:
\begin{align}
\hatLcal_i(h)-\hatLcal_i(\hat h^*) 
&\leq \Lcal_i(h)-\Lcal_i(\hat h^*) + 2\Ccal_i(n_i,\Hcal,\deltaprime) < 0.
\label{eq:contradiction-dominance}
\end{align}
For the first inequality we applied Lemma~\ref{lem:generalization}
twice, and the second inequality follows from the assumption. 
However, \eqref{eq:contradiction-dominance} establishes 
that $h$ empirically strongly dominates $\hat h^*$ which is a 
contradiction to the assumption that $\hat h^*$ was empirically
Pareto-optimal.

\myparagraph{Proof of Theorem~\ref{thm:pareto}}
We again only study these case in which the dataset fulfills 
the inequalities of Lemma~\ref{lem:generalization}, so 
the results we achieve holds with probability at least 
$1-\delta$.

Statement a) is a consequence of Lemma~\ref{lem:generalization} 
and the definition of (empirical) Pareto-optimality.
Let $h^*\in\Hcal$ be Pareto-optimal. If it is also
empirically Pareto-optimal, inequality~\eqref{eq:pareto}
holds trivially with $\hat h^*=h^*$. 
Otherwise, there exists an empirically Pareto-optimal 
$\hat h^*$ that dominates $h^*$ with respect to the 
empirical objectives, \ie in particular $\hatLcal_i(\hat h^*)\leq \hatLcal_i(h^*)$ for all $i\in[N]$.
From this, we obtain for all $i\in[N]$, analogously to the proof of Theorem~\ref{thm:scalarization}b):
\begin{align}
\Lcal_i(\hat h^*) - \Lcal_i(h^*) 
&\leq 
\Lcal_i(\hat h^*) - \hatLcal_i(\hat h^*) + \hatLcal_i(h^*) - \Lcal_i(h^*) 
\leq 
2\Ccal_i(n_i,\Hcal,\deltaprime).
\end{align}

Before proving statement b) we introduce 
\emph{additively shifted objectives}. 
as the main tool.

\begin{definition}
For an objective $\Lcal(h)$ with empirical estimate 
$\hatLcal(h)$ and a constant $K$, we call $\Lcal^{+K}(h)=\Lcal(h)+K$ 
and $\hatLcal^{+K}(h)=\hatLcal(h)+K$ their $K$-additively 
shifted variants.
\end{definition}

Generalization and Pareto-optimality are unaffected by additive shifts.

\begin{lemma}\label{lem:shift}
a) For any constant $K$, if a generalization bound of the 
form~\eqref{eq:single-generalization} holds for an objective $\Lcal$ 
and its empirical estimate $\hatLcal$, then a bound with 
identical generalization term also holds for $\Lcal^{+K}$ 
and $\hatLcal^{+K}$.
b) For any constants $K_1,\dots,K_N$, a solution $h\in\Hcal$ is 
Pareto-optimal for $\Lcal_1,\dots,\Lcal_N$ if and only if it 
is Pareto-optimal for $\Lcal^{+K_1}_1,\dots,\Lcal^{+K_N}_1$.
The analogous relation holds for empirically Pareto-optimality.
\end{lemma}
The proofs are elementary: for a) the additive terms 
cancel out in the generalization bound. 
For b) Pareto-optimality depends only on the the relative 
order of objective values, which is not affected by 
additive shifts. 

\begin{lemma}\label{lem:chebyshev}
Let $h^*\in\Hcal$ be a Pareto-optimal 
solution with $\Lcal_i(h)>0$ for 
all $i\in[N]$. 
Then $h^*$ is a minimizer to the Chebyshev scalarization
$\Ucal^{(\infty)}_w(h) = \max_{i\in[N]}w_i\Lcal_i(h)$
with weights $w_i=\frac{1}{\Lcal_i(h^*)}$
for $i\in[N]$.
Furthermore, for any other minimizer, $h^{\dag}$, 
of the scalarization it holds that $\Lcal_i(h^{\dag})=\Lcal_i(h^*)$
for all $i\in[N]$.
The analogous result holds for empirically Pareto-optimal hypotheses.
\end{lemma}
\begin{proof}
We prove the lemma by contradiction. First, 
assume $h$ to be a hypothesis with strictly 
smaller value for the scalarization. 
By construction $w_i\Lcal_i(h^*)=1$ for all 
$i\in[N]$, therefore $w_i\Lcal_i(h)<1$ for 
all $i\in[N]$ must hold. 
This, however, would imply $\Lcal_i(h)<\Lcal_i(h^*)$
for all $i\in[N]$, 
which is impossible because $h^*$ is Pareto-optimal.
For $h^{\dag}$, we know $w_i\Lcal_i(h^{\dag})\leq 1$ 
and therefore $\Lcal_i(h^{\dag})\leq\Lcal_i(h^*)$
for all $i\in[N]$. 
Because of $h^*$'s Pareto-optimality, none of these
inequalities can be strict, which proves the statement.
The same line of arguments holds in the empirical situation.
\end{proof}

We now turn to the proof of Theorem~\ref{thm:pareto} b). 
Let $\hat h^*$ be an empirically Pareto-optimal solution for $\Lcal_1,\dots,\Lcal_N$.
For a more concise notation, we abbreviate $c_i=\Ccal_i(n,\Hcal,\deltaprime)$.

First, we consider the case where none of the objectives are trivially 
generalizing, \ie $c_i>0$ for all $i\in[N]$. 
By Lemma~\ref{lem:shift}, we know that $h^*$ is also empirically Pareto-optimal 
for the shifted objectives $\hatLcal^{+K_1}_1,\dots,\hatLcal^{+K_{N'}}_{N'}$ 
with 
\begin{align}
K_i:=Cc_i-\hatLcal_i(h^*) \quad\text{ for }
 C=2+\max_j[\frac{1}{c_j}(\hatLcal_j(h^*)-\min_h\hatLcal_j(h))]
\end{align}
An explicit calculation confirms that $\hatLcal^{+K_i}_i(h)\geq 2c_i>0$, 
which by assumption implies $\Lcal^{+K_i}_i(h)\geq c_i>0$, for all $i\in[N]$.
By Lemma~\ref{lem:chebyshev} we know that $h^*$ is a minimizer 
of the Chebyshev scalarization with weights $w_i=\frac{1}{\Lcal^{+K_i}_i(h^*)}=\frac{1}{Cc_i}$ 
for all $i\in[N']$.
Let $h^*$ be a minimizer of the scalarization 
of the true objectives with same weights $w_i$.
The assumption of \emph{ray completeness} together with Lemma~\ref{lem:chebyshev} implies
$w_1\Lcal^{+K_1}_1(h^*)=\dots=w_N\Lcal^{+K_N}_N(h^*) = \max_{j\in[N]} w_j\Lcal^{+K_j}_N(h^*)$.
The Chebyshev scalarization is a weighted $L^{(\infty)}$-norm and 
$1$-Lipschitz with respect to itself. Therefore, by Theorem~\ref{thm:scalarization}:
\begin{align}
\max_{j\in[N]}w_j\Lcal^{+K_j}_j(\hat h^*)
&\leq 
\max_{j\in[N]}w_j\Lcal^{+K_j}_j(h^*)
+ 2\max_{j\in[N]}w_jc_j
\end{align}
Consequently, we obtain the component-wise inequalities:
\begin{align}
\forall i\in[N]\qquad w_i\Lcal^{+K_i}_i(\hat h^*) &\leq w_i\Lcal^{+K_i}_i(h^*) + 2\max_{j\in[N]}w_jc_j
\end{align}
Now, inserting the definition $K_i$, subtracting $w_iK_i$ from both sides 
and dividing by $w_i$ we obtain
\begin{align}
\forall i\in[N]\qquad \Lcal_i(\hat h^*) &\leq \Lcal_i(h^*) + \frac{2}{w_i}\max_{j\in[N]}w_jc_j.
\end{align}
By construction, $w_jc_j=\frac{1}{C}$ for all $j\in[N]$.
Therefore, $\frac{2}{w_i}\max_{j\in[N]}w_jc_j=\frac{2Cc_i}{C}=2c_i$. 
Because $c_i=\Ccal_i(n_i,\Hcal,\deltaprime)$ this concludes the proof.

For the general situation, assume that there are $\Nnt$ non-trivially 
and $N-\Nnt$ trivially generalizing objectives. 
If $M=0$, then $\Lcal_i(h)=\hatLcal_i(h)$ for all $i=1,\dots,N$ 
and for all $h\in\Hcal$. Then, Pareto-optimal and empirically 
Pareto-optimal sets coincide, and $\hat h^*=h^*$ fulfills the 
statement of the theorem. 

Otherwise, assume without loss of generality that the objectives 
are ordered such that, $\Ccal_i(n_i,\Hcal,\deltaprime)>0$ for $i\in[\Nnt]$ and $\Ccal_i(n_i,\Hcal,\deltaprime')=0$ for $i\in\{\Nnt+1,\dots,N\}$.
Let $\Gcal = \big\{h \in\Hcal : \hatLcal_i(h)=\hatLcal_i(\hat h^*)\text{ for $i\in\{\Nnt+1,\dots,N\}$}\big\}$.
Note that also $\Gcal = \big\{h \in\Hcal : \Lcal_i(h)=\Lcal_i(h^*)\text{ for $i\in\{\Nnt+1,\dots,N\}$}\big\}$, 
because $\Lcal_{\Nnt+1},\dots,\Lcal_N$ are trivially generalizing.
$\Gcal$ is a subset of $\Hcal$ that is non-empty (because $\hat h^*\in\Gcal$). 
Consequently, the inequalities of Lemma~\ref{lem:generalization} 
and Theorem~\ref{thm:scalarization} hold also as statements for 
all $g\in\Gcal$ rather than $h\in\Hcal$.
Because $\hat h^*$ is empirically Pareto-optimal within $\Hcal$ with 
respect to $\hatLcal_1,\dots,\hatLcal_N$, it is also empirically 
Pareto-optimal in $\Gcal$ with respect to $\hatLcal_1,\dots,\hatLcal_{M}$.  
Applying the result from the case without trivially-generalizing 
objectives to this situation, we obtain that there exists $h^*\in\Gcal$ 
such that for all $i\in[\Nnt]$
\begin{align}
\Lcal_i(\hat h^*) &\leq \Lcal_i(h^*) + \Ccal_i(n_i,\Hcal,\deltaprime)
\label{eq:subset}
\end{align}

For $i\in\{\Nnt+1,\dots,N\}$, we have $\Lcal_i(\hat h^*)=\Lcal_i(h^*)$,
because $h^*\in\Gcal$. Consequently, inequality \eqref{eq:subset}
holds also for these (with $\Ccal_i(n_i,\Hcal,\deltaprime)=0$), which concludes the proof.

\end{document}